\documentclass[letterpaper, conference]{ieeeconf}
\IEEEoverridecommandlockouts
\usepackage{amsmath,amssymb,amsfonts,theorem}
\usepackage{graphicx,epsfig,color}
\usepackage{subfigure,psfrag}
\usepackage{algorithmic}
\graphicspath{{./Figs/}}

\newcommand{\until}[1]{\{1,\dots, #1\}}
\newcommand{\subscr}[2]{#1_{\textup{#2}}}

\newcommand{\setdef}[2]{\{#1 \, | \; #2\}}
\newcommand{\seqdef}[2]{\{#1\}_{#2}}

\newcommand{\map}[3]{#1: #2 \rightarrow #3}
\newcommand{\setmap}[3]{#1: #2 \rightrightarrows #3}
\newcommand{\union}{\operatorname{\cup}}
\newcommand{\intersection}{\ensuremath{\operatorname{\cap}}}
\newcommand{\intersect}{\ensuremath{\operatorname{\cap}}}

\renewcommand{\theenumi}{(\roman{enumi})}

\newcommand{\integernonnegative}{\ensuremath{\mathbb{Z}}_{\ge 0}}
\newcommand{\card}[1]{\left|#1\right|}  
\newcommand{\powerset}[1]{\mathcal{C}(#1)}
\newcommand{\ConnPart}{\mathcal{P}}  
\newcommand{\Hexp}{\subscr{\mathcal{H}}{expected}}   
\newcommand\Prob{\mathbb{P}}
\newcommand{\myset}{\setdef{(i,j)\in\until{N}^2}{i\not=j}}

\renewcommand\footnoterule{\vspace{5pt}\hrule width \linewidth height 0.4pt\vspace{5pt}}

\newcommand\oprocendsymbol{\hbox{$\square$}}
\newcommand\oprocend{\relax\ifmmode\else\unskip\hfill\fi\oprocendsymbol}

\def\E{\mathcal{E}}

\def\R{\mathbb{R}}

\def\G{\mathcal{G}}

\newcommand{\Cd}{\operatorname{Cd}}      

\newcommand{\argmin}{\mathop{\operatorname{argmin}}}

\newcommand{\realpositive}{\ensuremath{\mathbb{R}}_{>0}}
\newcommand{\realnonnegative}{\ensuremath{\mathbb{R}}_{\geq0}}
\renewcommand{\natural}{{\mathbb{N}}}

\renewcommand{\H}{\mathcal{H}}
\newcommand{\Hgeneric}{\subscr{\H}{multicenter}}

\newcommand{\Hone}{\subscr{\H}{1}}

\newtheorem{theorem}{Theorem}[section]
\newtheorem{proposition}[theorem]{Proposition}

\newtheorem{definition}[theorem]{Definition}
\newtheorem{lemma}[theorem]{Lemma}

{\theorembodyfont{\rmfamily} 
\newtheorem{remark}[theorem]{Remark}

}

\newcommand{\bigO}[1]{\ensuremath{ \mathcal{O} ( #1 )}}
\def\log{\mathop{\operatorname{log}}}

\title{\bf Pairwise Optimal Discrete Coverage Control for Gossiping Robots
 \thanks{This material is based upon work
    supported in part by NSF grant CNS-0834446, NSF grant IIS-0904501, and
    ARO MURI grant W911NF-05-1-0219.}}

\author{Joseph W. Durham \and Ruggero Carli \and Francesco Bullo%
  \thanks{J. W. Durham, R. Carli, and F. Bullo are with the Center for Control, Dynamical
    Systems and Computation, University of California at Santa Barbara,
    Santa Barbara, CA 93106, USA, {\tt\small
      \{joey|carlirug|bullo\}@engr.ucsb.edu}. }}

\begin{document}
 \maketitle \thispagestyle{empty}
\pagestyle{empty}
\begin{abstract}
  We propose distributed algorithms to automatically deploy a group of
  robotic agents and provide coverage of a discretized environment
  represented by a graph.  The classic Lloyd approach to coverage
  optimization involves separate centering and partitioning steps and
  converges to the set of centroidal Voronoi partitions.  In this work we
  present a novel graph coverage algorithm which achieves better
  performance without this separation while requiring only pairwise
  ``gossip'' communication between agents.  Our new algorithm provably
  converges to an element of the set of pairwise-optimal partitions, a
  subset of the set of centroidal Voronoi partitions.  We illustrate that
  this new equilibrium set represents a significant performance improvement
  through numerical comparisons to existing Lloyd-type methods.  Finally,
  we discuss ways to efficiently do the necessary computations.
\end{abstract}

\section{Introduction}

Coordinated networks of robots are already in use for environmental monitoring~\cite{EF-NEL:06}
and warehouse logistics~\cite{PRW-RdA-MM:08}.  
In the near future, improvements to the capabilities of autonomous robots will
enable robotic teams to revolutionize transportation and delivery of products to customers,
search and rescue, and many other applications.
All of these tasks share a common feature: the robots are asked to provide service
over a space.  The distributed {\em territory partitioning problem} for robotic
networks consists of designing individual control and communication laws such
that the team will divide a space into territories.  Typically, partitioning is
done so to optimize a cost function which measures the quality of service
provided by the team.  {\em Coverage control} additionally optimizes the positioning
of robots inside a territory. 

This paper describes a distributed coverage control algorithm for a network of robots 
to optimize the response time of the team to service requests in an 
environment represented by a graph.  
Optimality is defined with reference to a cost
function which depends on the locations of the agents and
geodesic distances in the graph.  
As with all multiagent coordination
applications, the challenge comes from reducing the communication
requirements: the proposed algorithm requires only ``gossip"
communication, that is, asynchronous and unreliable pairwise
  communication.


A broad discussion of partitioning and coverage control is presented in~\cite{FB-JC-SM:09} which builds on the classic work of Lloyd~\cite{SPL:82} on algorithms for optimal quantizer design through ``centering and partitioning."  The relationship between discrete and continuous coverage control laws based on Euclidean distances is studied in~\cite{CG-JC-FB:06o}.  Coverage control and partitioning of discrete sets are also related to the literature on the facility location or $k$-center problem~\cite{VVVa:01}.  Coverage control algorithms for non-convex environments are discussed in~\cite{MZ-CGC:08,CHC-MZ:08,LCAP-VK-RCM-GASP:08} while equitable partitioning is studied in~\cite{OB-OB-DK-QW:07}.  Other works considering decentralized methods for coverage control include~\cite{SY-MS-DR:09} and \cite{MS-DR-JJS:08}.

In~\cite{PF-RC-FB:08u-arxiv} the authors have showed how a group of robotic agents can optimize the partition of convex bounded subset of $\R^d$ using a Lloyd-type algorithm with pairwise ``gossip" communication: only one pair of regions is updated at each step of the algorithm.  This gossip approach to Lloyd optimization was extended in~\cite{JWD-RC-PF-FB:08z} to discretized non-convex environments more suitable for physical robots.


There are three main contributions of this paper.
First, we present a novel gossip coverage algorithm and prove it converges to
an element in the set of pairwise-optimal partitions in finite time.
This solution set is shown to be a strict subset of the set of centroidal Voronoi partitions, meaning the new algorithm has fewer local minima than Lloyd-type methods.
Second, through realistic Player/Stage simulations we show that the set of pairwise-optimal partitions
avoids many of the local minima which can trap Lloyd-type algorithms far from the global optimum.
Third, we discuss how to efficiently compute our new pairwise coverage optimization.


This paper is organized as follows.  Section~\ref{sec:problem} defines the domain and goal of our algorithm, while \ref{sec:lloyd} contains a review of Lloyd-type gossip algorithms from~\cite{JWD-RC-PF-FB:08z}.  
Section~\ref{sec:optimal} presents our new algorithm and its properties, while in \ref{sec:computation} we detail its computational requirements.  In Section~\ref{sec:simulations} we provide numerical results and comparisons to prior methods, and we end with concluding remarks in~\ref{sec:conclusions}.

In our notation, $\realnonnegative$ denotes the set of non-negative real numbers
and $\integernonnegative$ the set of non-negative integers.  Given a
set $X$, $\card{X}$ denotes the number of elements in $X$.
Given sets $X,Y$, their difference is $X\setminus Y=\setdef{x\in X}{x\notin Y}$.  A set-valued map, denoted by $\setmap{T}{X}{Y}$, associates to an
element of $X$ a subset of $Y$.

\section{Problem formulation}
\label{sec:problem}

The problem domain and goal we are considering is the same as that in~\cite{JWD-RC-PF-FB:08z}:
given a group of $N$ robotic agents with limited sensing and
communication capabilities, and a discretized environment, we want to
apportion the environment into smaller regions and assign one region to each agent. 
The goal is to optimize the quality of coverage, as measured by a cost functional
which depends on the current partition and the positions of the agents.

Let the finite set $Q$ be the discretized environment.  We
assume that the elements of $Q$, which can be thought of as locations, are connected by weighted
edges.
Let $G=(Q,E,w)$ be an (undirected) weighted graph with edge set
$E\subset Q\times Q$ and weight map $\map{w}{E}{\realpositive}$; we let
$w_{e}>0$ be the weight of edge $e$.
 We assume that $G$ is connected and think of the edge weights as
distances between locations.  

In any weighted graph $G$ there is a standard notion of distance between
vertices defined as follows.  A \emph{path} in $G$ is an ordered sequence of
vertices such that any consecutive pair is an edge of
$G$.  The \emph{weight of a path} is the sum of the weights of the
edges in the path.  Given vertices $h$ and $k$ in $G$, the
\emph{distance} between $h$ and $k$, denoted $d_G(h,k)$, is the weight of
the lowest weight path between them, or $+\infty$ if there is no path. If $G$ is connected, 
then the distance between any two vertices is finite.
By convention, $d_G(h,k)=0$ if $h=k$. Note
that $d_G(h,k)=d_G(k,h)$, for any $h,k \in Q$.

We will be partitioning $Q$ into $N$ connected subsets or
territories to be covered by individual agents.  To do so we need to
define distances on induced subgraphs of $G=(Q,E,w)$.  
Given $I\subset Q$, the \emph{subgraph induced by the
restriction of $G$ to $I$}, denoted $G\intersect{}I$, is the graph
with vertex set equal to $I$ and with edge set containing
all weighted edges of $G$ where both vertices belong to $I$, i.e. 
$(Q,E,w)\intersect{}I=(Q\intersect{}I,E\intersect{}(I\times{I}),w|_{I\times{I}})$.
The induced subgraph is a weighted graph with a notion of
distance between vertices: given $h,k\in I$, we write
$
  d_I(h,k) := d_{G\intersect{}I}(h,k).
$
Note that $d_I(h,k)\ge d_G(h,k).$

We then define a {\em connected subset of $Q$} as a subset
$S \subset Q$ such that $S\neq\emptyset$ and $G \intersect S$ is
connected, with $\powerset{Q}$ denoting the set of such subsets.
We can then define partitions of $Q$ into connected sets as follows.
\begin{definition}[Connected partitions]
\label{def:ConPartitions}
Given the graph $G=(Q,E,w),$ we define a {\em connected $N-$partition
of $Q$} as a collection $p=\{p_i\}_{i=1}^{N}$ of $N$ subsets of $Q$ such that
\begin{enumerate}
\item ${\bigcup_{i=1}^{N}p_i=Q}$;
\item $p_i\cap p_j=\emptyset$ if $i\neq j$;
\item $p_i\neq\emptyset$ for all $i\in \until{N}$;
\item $p_i\in \powerset{Q}$ for all $i\in \until{N}$.
\end{enumerate}
Let $\ConnPart$ to be the set of such partitions.
\end{definition}

\begin{remark}
When $\card{Q} \geq N$ and $G$ is connected, it is always possible to find a connected $N-$partition
of $Q$.  Pick $N$ unique vertices from $Q$ to seed the subsets, then
iteratively grow each subset by adding unclaimed neighboring vertices.
\end{remark}

For our gossip algorithms we need to introduce the notion of adjacent subsets. Two distinct connected subsets $p_i$, $p_j$ are said to be \emph{adjacent} if there are two vertices $q_i$, $q_j$ belonging, respectively, to $p_i$ and $p_j$ such that $(q_i, q_j) \in E$. Observe that if $p_i$ and $p_j$ are adjacent then $p_i \union p_j \in \powerset{Q}$.
Similarly, we say that robots $i$ and $j$ are adjacent or are neighbors if their subsets $p_i$ and $p_j$ are adjacent.
Accordingly we next provide the definition of an \emph{adjacency graph}, which we will use in the analysis of our gossip algorithms.
\begin{definition}[Adjacency graph]
  For $p\in \ConnPart$, we define
  the {\em adjacency graph} between regions of partition $p$ as
  $\G(p)=(\{1,\ldots,N\},\E(p))$, where $(i,j)\in \E(p)$ if $p_i$ and $p_j$ are adjacent.
\end{definition}

On $Q$, we define a \emph{weight function} to be a bounded positive function $\map{\phi}{Q}{\realpositive}$ which assigns a positive weight to each element of $Q$.
Given $p_i\in \powerset{Q}$, we define the {\em one-center function} $\map{\Hone}{p_i}{\realnonnegative}$ as
$$\Hone(h; p_i)={\sum_{k\in p_i} {d_{p_i}(h,k)\phi(k)}}.$$

A technical assumption is then needed to define the \emph{generalized centroid} of a connected subset. In what follows, we assume that a {\em total order} relation, $<$, is defined on $Q$: hence, we can also denote $Q=\until{\card{Q}}$.


\begin{definition}[Centroid]\label{def:Centroid}
Let $Q$ be a totally ordered set, and $p_i\in \powerset{Q}$. We define generalized centroid of $p_i$ as
\begin{align*}
\Cd(p_i):=\min\{\argmin_{h\in p_i} \Hone(h;p_i)\}.
\end{align*}
\end{definition}
In subsequent use we will drop the word ``generalized" for brevity. Note that with this definition the centroid is well-defined, and also that the centroid of a region always belongs to the region.
With a slight notational abuse, we define $\map{\Cd}{\ConnPart}{Q^N}$ as the map which associates to a partition the vector of the centroids of its elements.

With these notions we can introduce the \emph{multicenter function} $\map{\Hgeneric}{\ConnPart\times Q^N}{\realnonnegative}$ defined by
%
$$\Hgeneric(p,c)=\sum_{i=1}^{N} \Hone(c_i; p_i).$$
Our motivation for using this cost function is to optimize the response time of the robots to a task appearing randomly in $Q$ according relative weights $\phi$.  We aim to minimize $\Hgeneric$ with respect to both the partitions $p$ and the points $c$ so as to minimize the expected distance from this random vertex to the centroid of the partition the vertex is in.

Among the ways of partitioning $Q$, there is one which is worth special attention. Given points $c\in Q^N$ such that if $i\neq j$, then $c_i\neq c_j$, the partition $p \in \ConnPart$ is said to be a \emph{Voronoi partition of Q generated by c} if, for each $p_i$ and all $k \in p_i$, we have
$c_i\in p_i$
and
$d_G(k,c_i) \le d_G(k,c_j)$, $\forall j\neq i$.
\begin{proposition}[Properties of multicenter function]\label{prop:optimal-for-Hgeneric}
Let $p\in \ConnPart$, $c\in Q^N$, and let $p^*$ be a Voronoi partition generated by $c$. Then
\begin{align*}
\Hgeneric(p,\Cd(p)) &\le \Hgeneric(p,c), \\ 
\Hgeneric(p^*,c) & \le \Hgeneric(p,c). 
\end{align*}
\end{proposition}

These statements motivate the following definition: a partition $p\in \ConnPart$ is a \emph{centroidal Voronoi partition} if $p$ is a Voronoi partition generated by $\Cd(p)$. Based on the multicenter function, we define  $\map{\Hexp}{\ConnPart}{\realnonnegative}$
by
\begin{align*}
\Hexp(p)&=\Hgeneric(p,\Cd(p))\\
&=\sum_{i=1}^N\sum_{x\in p_i} d_{p_i}(x,\Cd(p_i))\phi(x).
\end{align*}

Observe that $\Hexp$ has the following property as an immediate consequence of Proposition~\ref{prop:optimal-for-Hgeneric}: given $p\in \ConnPart$, if $p^*$ is a Voronoi partition generated by $\Cd(p)$ then
$$
\Hexp(p^*)\leq \Hexp(p).
$$

We are now in a position to state the goal of our algorithm: solving the
optimization problem
$$\min_{p\in \ConnPart} \Hexp(p),$$ 
using only pairwise territory
exchanges between agents along the edges of $\E(p)$.
In the literature this optimization is typically studied with
either centralized control or with synchronous and reliable communication between
several agents, see~\cite{FB-JC-SM:09} and~\cite{LCAP-VK-RCM-GASP:08}.  We believe these communication requirements are
unrealistic for deployed robotic networks, and thus are interested in
solutions requiring only pairwise gossip communication as first explored in~\cite{PF-RC-FB:08u-arxiv} and~\cite{JWD-RC-PF-FB:08z}.

\section{Lloyd-type gossip graph coverage}
\label{sec:lloyd}

In this Section we briefly review the \emph{discretized Lloyd-type gossip coverage algorithm} proposed in \cite{JWD-RC-PF-FB:08z}.

\medskip\footnoterule\vspace{.75\smallskipamount}
\noindent\hfill\textbf{Lloyd-type Gossip Graph Coverage Algorithm}\hfill\vspace{.75\smallskipamount}
\footnoterule\vspace{.75\smallskipamount}

\noindent At each time $t\in\integernonnegative$, each agent $i\in\until{N}$ maintains in memory
a connected subset $p_i(t)$.  The collection $p(0)=\{p_1(0),\dots,p_N(0)\}$ is
an arbitrary connected $N-$partition of $Q$.  At each $t\in\integernonnegative$ a
communicating pair, say $(i,j)\in \E(p(t))$, is selected by a deterministic
or stochastic process to be determined. Assume that $i<j$. Every agent $k\not\in\{i,j\}$ sets
$p_k(t+1)=p_k(t)$, while $i$ and $j$ perform the following:

\begin{algorithmic}[1]
  \STATE agent $i$ transmits its subset $p_i(t)$ to $j$ and vice-versa%
  \STATE both agents compute centroids $\Cd(p_i)$ and $\Cd(p_j)$, set $u = p_i \cup p_j$, and the sets
  \begin{align*}
& W_{i\to j}=\left\{x\in p_i : d_{u}(x, \Cd(p_j))< d_{u}(x, \Cd(p_i))\right\} \\
& W_{j\to i}=\left\{x\in p_j : d_{u}(x, \Cd(p_i))< d_{u}(x, \Cd(p_j))\right\} \\
& W_{i\cong j}=\left\{x\in p_i \cup p_j :  d_{u}(x, \Cd(p_i))= d_{u}(x, \Cd(p_j))\right\}
  \end{align*}
   \IF{$W_{i\to j}\cup W_{j\to i}= \emptyset$}
  \STATE $\! p_i(t+1) := p_i(t)$ and $p_j(t+1) := p_j(t)$
  \ELSE \STATE
  $\! p_i(t+1) := \left(\left(p_i \setminus W_{i\to j}\right) \cup W_{j \to i}\right) \cup W_{j\cong i}, $ \\
  $\! p_j(t+1) := \left(\left(p_j \setminus W_{j\to i}\right) \cup W_{i \to j}\right) \setminus \left( W_{j\cong i} \intersection p_j \right) $
  \ENDIF
\end{algorithmic}
\vspace{.5\smallskipamount}\footnoterule\smallskip

Observe that $W_{i \to j}$ (resp. $W_{j \to i}$) contains the cells of $p_i$ (resp. $p_j$) which are closer to $\Cd(p_j)$ (resp. $\Cd(p_i)$), whereas $W_{i\cong j}$ represents the set of the tied cells. In other words, when two robots exchange territory using this Lloyd-type algorithm, their updated regions $\left\{p_{i}(t+1),p_{j}(t+1)\right\}$ are the Voronoi partition of the set $p_i(t)\cup p_j(t)$ generated by the centroids $\Cd(p_i)$ and $\Cd(p_j)$, with all tied cells assigned to the agent with the lower index.

Now, for any pair $(i,j)\in\until{N}^2$, $i\not=j$, we define the map
$\setmap{T_{ij}}{ \ConnPart}{ \ConnPart}$ as
\begin{equation*}
  T_{ij}(p) =
  \begin{cases}
    p, \qquad \text{if}\,\,\, (i,j)\notin \E(p)\,\,\, \text{or}\,\,\, W_{i\to j}\cup W_{j\to i}= \emptyset\\
    \{p_1,\dots,\widehat{p}_i,\dots,\widehat{p}_j, \dots,p_N\},
    \qquad \text{otherwise},
  \end{cases}
\end{equation*}
where $\widehat{p}_i = \left(\left(p_i \setminus W_{i\to j}\right) \cup W_{j \to i}\right) \cup W_{j\cong i},$ and $\widehat{p}_j = \left(\left(p_j \setminus W_{j\to i}\right) \cup W_{i \to j}\right) \setminus \left( W_{j\cong i} \intersection p_j \right).$

The dynamical system on the space of partitions is therefore described by,
for $t\in\integernonnegative$,
\begin{equation}
  \label{eq:OurAlgo-ij}
  p(t+1)=T_{ij}(p(t)), \quad \text{ for some } (i,j)\in \E(p(t)),
\end{equation}
together with a rule for which edge $(i,j)$ is selected at each time.
We also define the set-valued map $\setmap{T}{\ConnPart}{\ConnPart}$ by
\begin{equation}\label{eq:AlgoT}
T(p)=\setdef{T_{ij}(p)}{(i,j)\in \E(p)}.
\end{equation}

In Theorem~\eqref{th:T-persistent-gossip-lloyd} we summarize the convergence properties of the Lloyd-type discretized gossip coverage algorithm. To do so, we need the following definition. 
\begin{definition}[Uniform and random persistency]
  \label{def:persistency}
  Let $X$ be a finite set.
  \begin{enumerate}
  \item A map $\map{\sigma}{\integernonnegative}{X}$ is \emph{uniformly
      persistent} if there exists a duration $\Delta\in\natural$ such that,
    for each $x\in{X}$, there exists an increasing sequence of times
    $\seqdef{t_k}{k\in\integernonnegative}\subset\integernonnegative$
    satisfying $t_{k+1}-t_k\leq\Delta$ and $\sigma(t_k)=x$ for all
    $k\in\integernonnegative$.
  \item A stochastic process $\map{\sigma}{\integernonnegative}{X}$ is
    \emph{randomly persistent} if there exists a probability $p\in{]0,1[}$
 such that, for each $x\in{X}$ and
    for each $t\in\integernonnegative$    
    \begin{equation*}
      \Prob\big[\sigma(t+1)= x \,|\,  \sigma(t),\dots,\sigma(1)\big] \geq p.
    \end{equation*}
  \end{enumerate}
\end{definition}

\begin{theorem}[Convergence under persistent gossip]
  \label{th:T-persistent-gossip-lloyd}
  Consider the Lloyd-type discretized gossip coverage algorithm $T$ and the evolutions
  $\map{p}{\integernonnegative}{\ConnPart}$ defined by
  \begin{equation*}
    p(t+1) = T_{\sigma(t)}(p(t)), \quad \text{for } t\in\integernonnegative,
  \end{equation*}
  where $\map{\sigma}{\integernonnegative}{\myset}$ is either a deterministic
  map or a stochastic process.  Then the following statements hold:
  \begin{enumerate}
  \item if $\sigma$ is a uniformly persistent map, then any
    evolution $p$ converges in a finite number of steps to a centroidal Voronoi partition;
    and

  \item if $\sigma$ is a randomly persistent stochastic process, then any
    evolution $p$  converges almost
    surely in a finite number of steps to a centroidal Voronoi partition.
  \end{enumerate}
\end{theorem}

Convergence to a centroidal Voronoi partition puts this pairwise
optimization algorithm in the same class with centralized Lloyd methods.
However, there can be a great number of possible centroidal Voronoi partitions
for a given discretized environment and, in our experience, the algorithm 
too frequently converges to a suboptimal solution.
This issue motivates the developments in the following Section.

\section{Pairwise-optimal gossip coverage algorithm}
\label{sec:optimal}

In this section we present a novel gossip graph coverage algorithm which improves on the performance achievable by the Lloyd-type algorithm reviewed in the previous Section. First we introduce the following notational definition.

Given $p\in \ConnPart$ and distinct components $p_i$ and $p_j$, let $P_{p_i\cup p_j}$ denote the set of all distinct pairs of vertices in $p_i\cup p_j$. We assume that the elements in $P_{p_i \cup p_j}$ are sorted lexicographically. In formal terms, if 
$p_{i} \cup p_{j}=\left\{h_1, \ldots, h_{|p_i\cup p_j|}\right\}$
where $h_s< h_{s+1}$ for $s=\until{|p_i\cup p_j|-1}$, then  
\begin{align*}
P_{p_i\cup p_j}=&\left\{(h_1,h_2), (h_1,h_3),\ldots,(h_1, h_{|p_i\cup p_j|}), (h_2, h_3), \ldots \right.\\
&\qquad\qquad\qquad\left. \ldots, (h_{|p_i\cup p_j|-1}, h_{|p_i\cup p_j|})\right\}.
\end{align*}
Notice that $|P_{p_i\cup p_j}|=\frac{|p_i\cup p_j|\, (|p_i\cup p_j|-1)}{2}$. 

The \emph{pairwise-optimal discretized gossip coverage algorithm} is defined as follows. 

\medskip\footnoterule\vspace{.75\smallskipamount}
\noindent\hfill\textbf{Gossip Optimal Graph Coverage Algorithm}\hfill\vspace{.75\smallskipamount}
\footnoterule\vspace{.75\smallskipamount}

\noindent At each time $t\in\integernonnegative$, each agent $i\in\until{N}$
maintains in memory a connected subset $p_i(t)$.  The collection
$p(0)=\{p_1(0),\dots,p_N(0)\}$ is an arbitrary connected $N-$partition of
$Q$.  At each $t\in\integernonnegative$ a communicating pair, say $(i,j)\in
\E(p(t))$, is selected by a deterministic or stochastic process to be
determined. Assume that $i<j$. Every agent $k\not\in\{i,j\}$ sets
$p_k(t+1)=p_k(t)$, while $i$ and $j$ perform the following:

\begin{algorithmic}[1]
  \STATE agent $i$ transmits its subset $p_i(t)$ to $j$ and vice-versa%
  \STATE both agents compute the set $u = p_i \cup p_j$, the function $\mathcal{D}: P_{u} \to \realpositive$
  $$
  \mathcal{D}((a,b))=\sum_{h\in u} \min \left\{d_{u}(h, a), d_{u}(h, b)\right\},
  $$
  and find the set 
  $\mathcal{S}_{\mathcal{D}}=\mathop{\argmin}_{(h_s,h_r) \in P_{u}} \left\{\mathcal{D}(h_s, h_r)\right\}$
  \STATE assuming $\mathcal{S}_{\mathcal{D}}$ is lexicographically ordered, both agents pick the first pair in $\mathcal{S}_{\mathcal{D}}$, say $(a^*,b^*)$
  \STATE both agents compute the sets 
    \begin{align*}
      & W_{a^*}=\left\{x\in u : d_{u}(x, a^*) \leq d_{u}(x, b^*)\right\} \\
      & W_{b^*}=\left\{x\in u : d_{u}(x, b^*) < d_{u}(x, a^*)\right\}
  \end{align*}
  \IF {$\Hone(\Cd(W_{a^*});W_{a^*}) + \Hone(\Cd(W_{b^*});W_{b^*}) <$ \\ 
  	   $\Hone(\Cd(p_i(t));p_i(t)) + \Hone(\Cd(p_j(t));p_j(t))$}
      \STATE $p_i(t+1)=W_{a^*}, \quad p_j(t+1)=W_{b^*}$
  \ELSE
      \STATE $p_i(t+1)=p_i(t), \quad p_j(t+1)=p_j(t)$
  \ENDIF
\end{algorithmic}
\vspace{.5\smallskipamount}\footnoterule\smallskip

\begin{remark}
The lexicographic rule for picking $(a^*,b^*)$ used here makes the dynamical system represented by the algorithm deterministic and well-defined.  In practice any pair in $\mathcal{S}_{\mathcal{D}}$ can be chosen.
\end{remark}

The following Proposition characterizes a desirable property of the regions $W_{a^*}$ and $W_{b^*}$. 

\begin{proposition}\label{prop:Wconnected}
The regions $W_{a^*}$ and $W_{b^*}$ as defined in Gossip Optimal Graph Coverage Algorithm, are connected.
\end{proposition}

Thanks to Proposition~\ref{prop:Wconnected} we have that, for any pair $(i,j)\in\until{N}^2$, $i\not=j$, we can define the map
$\map{T_{ij}}{\ConnPart}{\ConnPart}$ by
\begin{equation*}
  T_{ij}(p) =
  \begin{cases}
    p, \qquad \qquad  \text{if}\,\,\,\,\,\, (i,j)\notin \E(p)\\
    \{p_1,\dots,\widehat{p}_i,\dots,\widehat{p}_j, \dots,p_N\},
    \qquad \text{otherwise},
  \end{cases}
\end{equation*}
where $\widehat{p}_i = W_{a^*}$ and $\widehat{p}_j = W_{b^*}$. 
Therefore, the dynamical system on the space of partitions can again be described in compact form by  \eqref{eq:OurAlgo-ij} and~\eqref{eq:AlgoT} but with this new $T_{ij}$. 

To establish the convergence properties of the \emph{pairwise-optimal discretized gossip coverage algorithm} we need the following definition.

\begin{definition}
A partition $p\in \ConnPart$ is said to be a \emph{pairwise-optimal partition} if, for every $(i,j)\in \E(p)$
\begin{align*}
&\Hone(\Cd(p_i); p_i)+\Hone(\Cd(p_j); p_j)=\\
&\qquad  \min_{a, b \in p_i \cup p_j} \left\{\sum_{k\in p_i \cup p_j} \min \left\{d_{p_i \cup p_j} (a,k), d_{p_i \cup p_j} (b,k)\right\}\right\}
\end{align*}
\end{definition}

The following Proposition characterizes the set of the pairwise-optimal partitions. 

\begin{proposition}\label{prop:OptPair}
Let $p \in \ConnPart$ be a \emph{pairwise-optimal partition}. Then $p$ is also a \emph{centroidal Voronoi partition}.
\end{proposition}


We are now ready to state the main result of this paper.

\begin{theorem}[Convergence under persistent gossip]
  \label{th:T-persistent-gossipOpt}
  Consider the pairwise-optimal discretized gossip coverage algorithm $T$ and the evolutions
  $\map{p}{\integernonnegative}{\ConnPart}$ defined by
  \begin{equation*}
    p(t+1) = T_{\sigma(t)}(p(t)), \quad \text{for } t\in\integernonnegative,
  \end{equation*}
  where $\map{\sigma}{\integernonnegative}{\myset}$ is either a deterministic
  map or a stochastic process.  Then the following statements hold:
  \begin{enumerate}
  \item if $\sigma$ is a uniformly persistent map, then any
    evolution $p$ converges in a finite number of steps to a pairwise-optimal partition;
    and

  \item if $\sigma$ is a randomly persistent stochastic process, then any
    evolution $p$  converges almost
    surely in a finite number of steps to a pairwise-optimal partition.
  \end{enumerate}
\end{theorem}

  In contrast to the Lloyd-type algorithm in
  Section~\ref{sec:lloyd}, this new pairwise-optimal algorithm converges to
  an element in the set of pairwise-optimal partitions.  This result is a
  useful improvement as the set of pairwise-optimal partitions is a subset
  of the set of centroidal Voronoi partitions, so that the new algorithm
  can avoid many of the local minima in which Lloyd-type methods get stuck.
  We will demonstrate this improvement in Section~\ref{sec:simulations}.

\section{Convergence proofs}
\label{sec:proofs}

In this Section we provide proofs of the results stated in the previous Section, starting with Proposition \ref{prop:Wconnected}. 

\begin{proof}[Proposition \ref{prop:Wconnected}]
To prove the statement of the Proposition we show that, given $x\in W_{b^*}$, any shortest path in $p_i \cup p_j$ connecting $x$ to $b^*$ completely belongs to $W_{b^*}$. We proceed by contradiction. Let $s_{x,b^*}$ denote a shortest path in $p_i\cup p_j$ connecting $x$ to $b^*$ and let us assume that there exists $m\in s_{x,b^*}$ such that $m\in W_{a^*}$. For $m$ to be in $b^*$ means that $d_{p_i\cup p_j}(m, b^*) < d_{p_i\cup p_j}(m, a^*)$. This implies that
\begin{align*}
d_{p_i\cup p_j}(x, a^*)&=d_{p_i\cup p_j}(m,a^*)+d_{p_i \cup p_j}(x,m)\\
&> d_{p_i\cup p_j}(m,b^*)+d_{p_i\cup p_j}(x,m)\\
&\geq d_{p_i \cup p_j}(x,b^*).
\end{align*}
This is a contradiction for $x\in b^*$. Similar considerations hold for $W_{a^*}$.
\end{proof}

To prove Proposition \ref{prop:OptPair}  we need to review the notion of a \emph{centroidal Voronoi in pairs} partition~\cite{JWD-RC-PF-FB:08z}.

\begin{definition}
Consider the connected graph $G=(Q,E,w).$. Let $p\in \ConnPart$ and let $\mathcal{G}(p)$ be the associated adjacency graph. Then
$p\in \ConnPart$ is said to be \emph{centroidal Voronoi in pairs} if, for any $i\in \until{N}$ we have that
\begin{equation}
d_{p_i\cup p_j} (k, \Cd(p_i))\leq d_{p_i\cup p_j} (k, \Cd(p_j))
\end{equation}
for all $k\in p_i$ and for all $j$ such that $(i,j) \in \mathcal{E}(p)$.
\end{definition}

The following Lemma \cite{JWD-RC-PF-FB:08z} states a useful equivalence property between the set of centroidal Voronoi in pairs partitions and the set of centroidal Voronoi partitions.

\begin{lemma}\label{lem:cVIPiffcV}
  A partition $p\in\ConnPart$ is centroidal Voronoi in pairs if and only if it is centroidal Voronoi.
\end{lemma} 

We are now ready to prove Proposition~\ref{prop:OptPair}.

\begin{proof}[Proposition \ref{prop:OptPair}]
We will show that if a partition $p\in \ConnPart$ is a pairwise-optimal partition then $p$ is also centroidal Voronoi in pairs. Then from Lemma~\ref{lem:cVIPiffcV} it will follow that $p$ is also centroidal Voronoi. 

We proceed by contradiction. Let $p\in \ConnPart$ be a pairwise-optimal partition, let $p_i$ and $p_j$ be two adjacent regions, and let us assume that there exists $x\in p_i$ such that 
\begin{equation}\label{eq:pi-pj}
d_{p_i \cup p_j}(x,\Cd(p_i))> d_{p_i \cup p_j}(x, \Cd(p_j)).
\end{equation}
Let $s_{x,\Cd(p_j)}$ denote a shortest path in $p_i\cup p_j$ connecting $x$ to $\Cd(p_j)$. Without loss of generality let us assume that $s_{x,\Cd(p_j)}=\left(x, y_1, y_2, \ldots, y_m, \Cd(p_j)\right)$, for a suitable $m$-tuple in $p_i \cup p_j$. Now let $s^{(i)}_{x, \Cd(p_j)}$ denote the subset of $s_{x,\Cd(p_j)}$ consisting of only the vertices belonging to $p_i$, i.e.,
$$
s^{(i)}_{x, \Cd(p_j)}=\left\{y\in s_{x, \Cd(p_j)}| y\in p_i\right\}.
$$
Observe that from~\eqref{eq:pi-pj} it follows that
$$
d_{p_i \cup p_j}(y,\Cd(p_i))> d_{p_i \cup p_j}(y, \Cd(p_j))
$$
for all $y\in s^{(i)}_{x, \Cd(p_j)}$. Therefore, by letting
$$
\bar{p}_i=p_i \setminus s^{(i)}_{x, \Cd(p_j)} \qquad \text{and} \qquad \bar{p}_j=p_j \cup s^{(i)}_{x, \Cd(p_j)} 
$$
we have that 
\begin{align*}
&\Hone(\Cd(\bar{p}_i); \bar{p}_i)+ \Hone(\Cd(\bar{p}_j); \bar{p}_j)\\
&\qquad\qquad\qquad \leq \Hone(\Cd(p_i); \bar{p}_i)+ \Hone(\Cd(p_j); \bar{p}_j)\\
&\qquad\qquad\qquad < \Hone(\Cd(p_i); p_i)+ \Hone(\Cd(p_j); p_j).
\end{align*}
This last inequality contradicts the fact that $p$ is a pairwise-optimal partition.
\end{proof}

To provide the proof of Theorem \ref{th:T-persistent-gossipOpt} we need the following intermediate result.

\begin{lemma}\label{prop:Tdecr}
Let $\setmap{T}{\ConnPart}{\ConnPart}$ be the point-to-set map defined for the Gossip Optimal Graph Coverage Algorithm.
If $p(t)\in\ConnPart$ and $p(t+1)\in T(p(t))\setminus p(t)$, then $\Hexp(p(t+1)) < \Hexp(p(t))$.
\end{lemma}
\begin{proof}
Without loss of generality assume that $(i,j)$ is the pair selected at time $t$. For simplicity let us denote $p(t+1)$ and $p(t)$, respectively, by $p^+$ and $p$.
Then
\begin{align*}
&\Hexp(p^+)- \Hexp(p)\\
&\qquad\qquad\qquad=  \Hone(\Cd(p_i^+); p_i^+)+ \Hone(\Cd(p_j^+); p_j^+)\\
&\qquad\qquad\qquad\qquad-(\Hone(\Cd(p_i); p_i)+ \Hone(\Cd(p_j); p_j)).
\end{align*}
According to the definition of the Gossip Optimal Graph Coverage Algorithm we have that if $p_i^+\neq p_i$, $p_j^+\neq p_j$, then 
\begin{align*}
&\Hone(\Cd(p_i^+); p_i^+)+ \Hone(\Cd(p_j^+); p_j^+)\\
&\qquad\qquad <\Hone(\Cd(p_i); p_i)+ \Hone(\Cd(p_j); p_j)
\end{align*}
from which the thesis follows.
\end{proof}

\begin{figure*}[t]
\centering
\subfigure
{
    \includegraphics[width=2.0in]{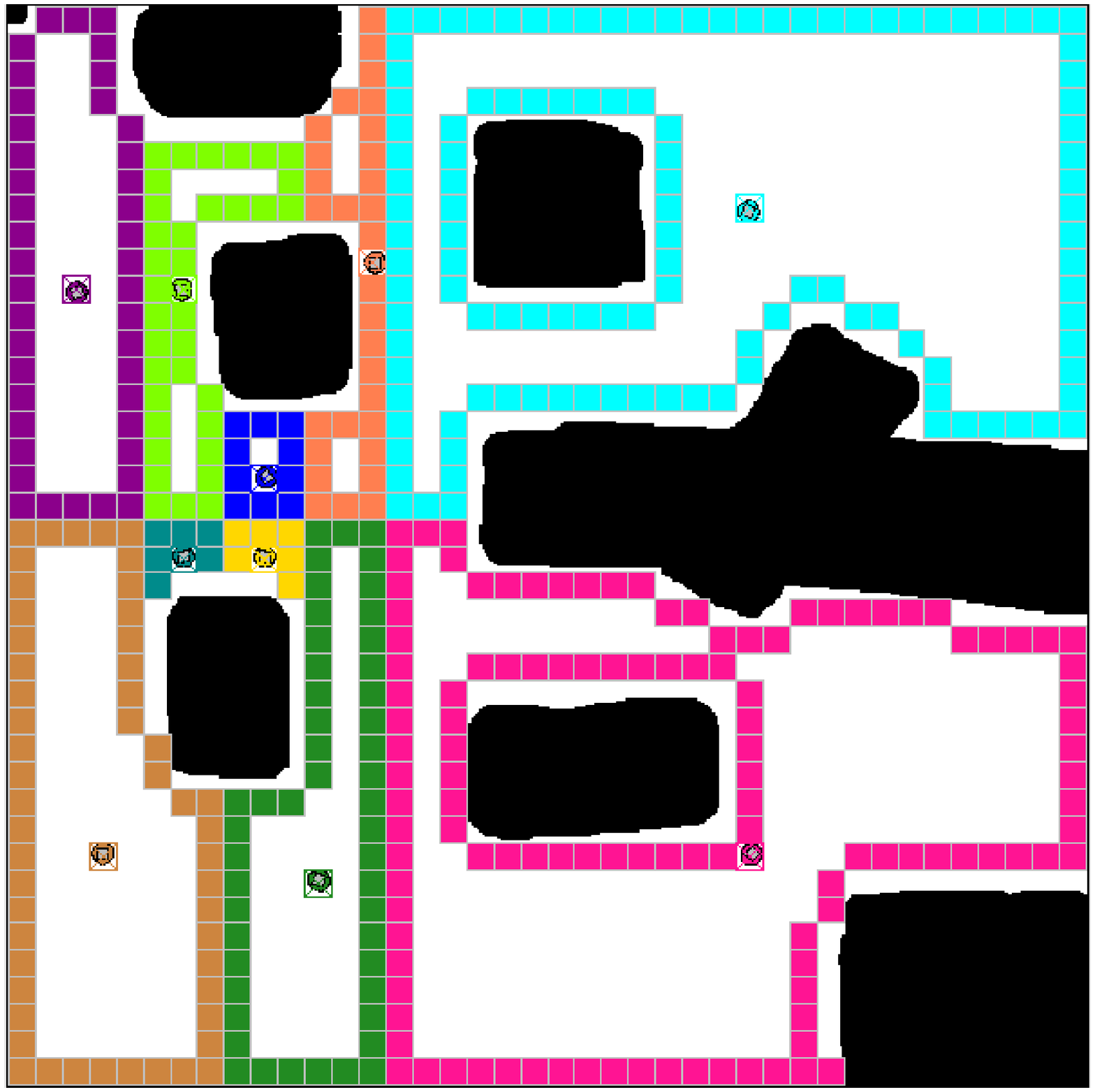}
    \label{fig:cave_initial}
}
\subfigure
{
    \includegraphics[width=2.0in]{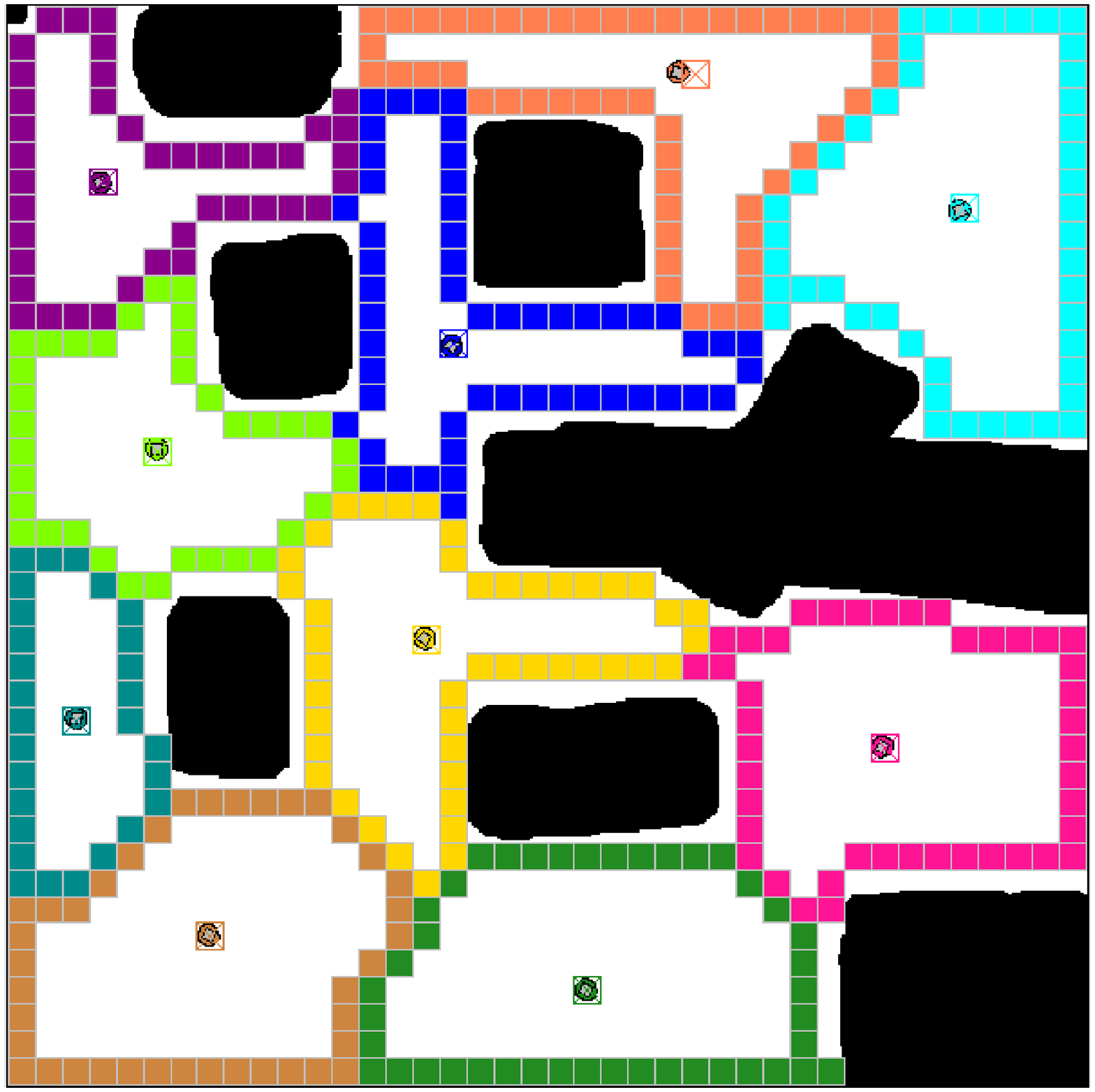}
    \label{fig:cave_lloyd}
}
\subfigure
{
    \includegraphics[width=2.0in]{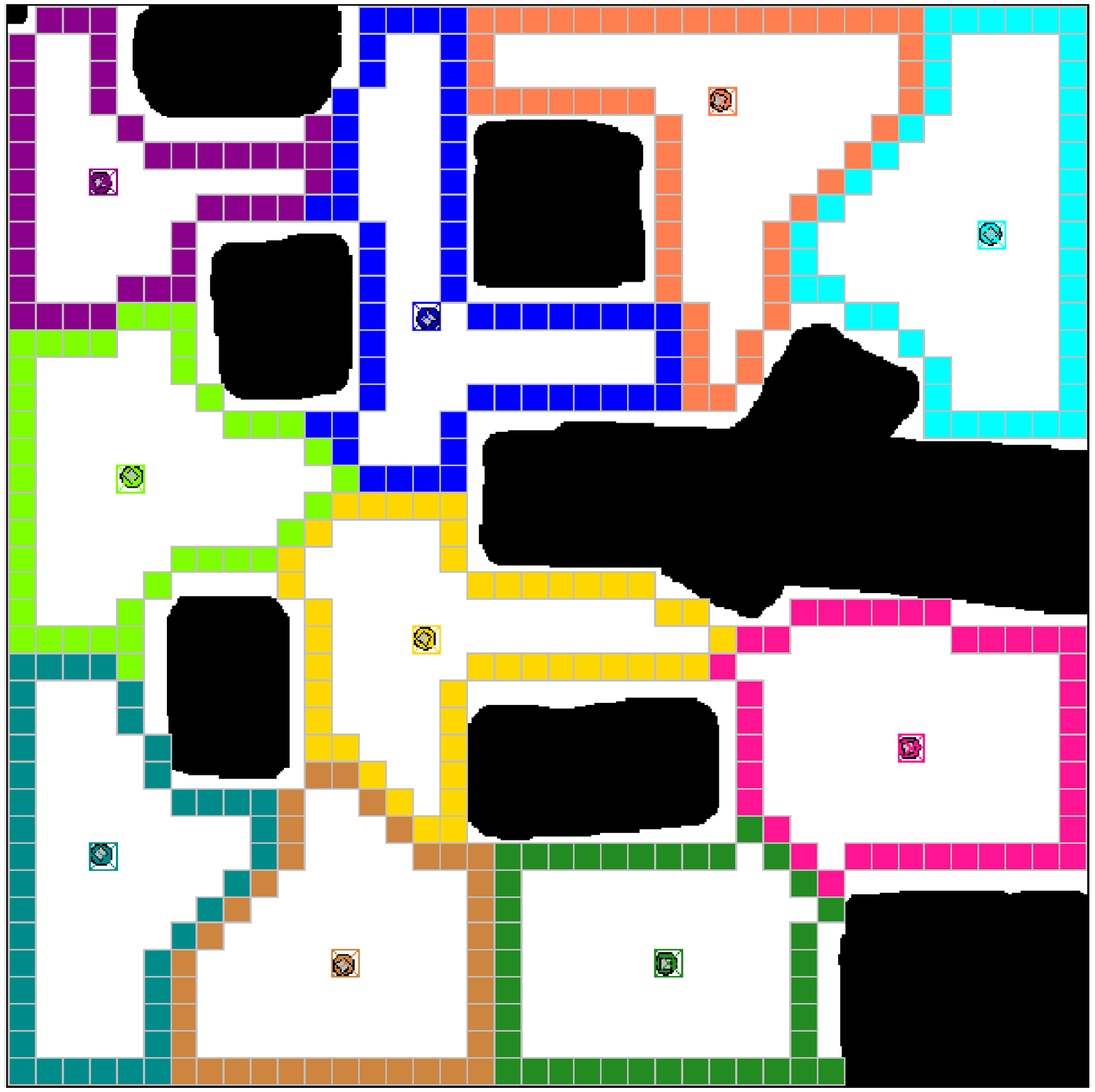}
    \label{fig:cave_optimal}
}
\caption{At left, ten robots are positioned at the centroids of their initial
partitions in a non-convex environment.  The boundary of 
each agent's partition is drawn in a different color.  The middle shows the final solution of
a centralized Lloyd optimization with a total cost of $624.1m$.  At right is a final solution for the pairwise-optimal discretized gossip coverage algorithm starting from the final
centralized Lloyd solution, with an improved cost of $610.3m$.}
\label{fig:sim_cave}
\end{figure*}

Now we can prove the main result of this paper.

\begin{proof}[Theorem \ref{th:T-persistent-gossipOpt}]
The proof of point (i) of the Theorem is based on 
on verifying the assumptions (i), (ii), (iii), (iv) of Theorem~4.3 from~\cite{PF-RC-FB:08u-arxiv}.  

We start with the following topological consideration.
Let $A,B\subset Q$, $A\Delta B$ be their symmetric difference, and let
$\card{A}$ be the ``points counting measure" of $A$, that is the number of
elements of $A$. 
Then if we define $d_\Delta(A,B)=\card{A\Delta B},$ we
have that $d_\Delta$ is a distance on the set of the subsets of $Q$, which we
denote $2^Q$. 
Since the points counting measure takes integer values, $2^Q$ a discrete topological space.
Next, consider the product
space $(2^Q)^N$ and let $q=\left\{q_1,\ldots,q_N\right\}$ and
$\bar{q}=\left\{\bar{q}_1,\ldots,\bar{q}_N\right\}$ be two its
elements. Then, thanks to results valid for product topological spaces, we
have that the function
$d_\Delta^{(N)}(q,\bar{q})=\sum_{i=1}^Nd\left(q_i,\bar{q}_i\right)$ is a
distance on $(2^Q)^N$. Hence, $(2^Q)^N$ is also a discrete topological
space. Observe that $\ConnPart$ is a subset of $(2^Q)^N$. Since $2^Q$,
$(2^Q)^N$ and $\ConnPart$ are finite, they are also trivially compact.
Moreover, since the algorithm $\setmap{T}{\ConnPart}{\ConnPart}$ is
well-defined, we have that $\ConnPart$ is strongly positively
invariant. This fact implies that assumption (i) of Theorem~4.3 from~\cite{PF-RC-FB:08u-arxiv} is satisfied.

From Lemma~\ref{prop:Tdecr} it follows that assumption (ii) is also satisfied, with $\Hexp$ playing the role of the function $U$.  

Next, the fact that all maps from a discrete space are continuous implies the continuity of the maps $T_{ij}$ and $\Hexp$, thus guaranteeing the validity of assumption (iii). Assumption (iv) holds true for the hypothesis made in the statement of Theorem~\ref{th:T-persistent-gossipOpt} that $\sigma$ is a uniformly persistent map.
Hence we are in the position to apply Theorem~4.3 from~\cite{PF-RC-FB:08u-arxiv}, and conclude convergence to the intersection of the equilibria of the maps $T_{ij}$, which, according to the definition of $T_{ij}$ coincides with the set of the pairwise-optimal partitions.
Hence, since that set is finite, we can argue that the system converges in finite time to one pairwise-optimal partition.

A proof of (ii) in Theorem~\ref{th:T-persistent-gossipOpt} can be made in a similar way using Theorem~4.5 from~\cite{PF-RC-FB:08u-arxiv} inplace of 4.3.
\end{proof}

\section{Scalability Properties}
\label{sec:computation}

In this Section we discuss how to compute the optimal new centroids $(a^*,b^*) \in P_{p_i\cup p_j}$ for the Gossip Optimal Graph Coverage Algorithm, as well as a sampling method to greatly reduce the computational complexity.  Determining $(a^*,b^*)$ requires an exhaustive search over all possible pairs of vertices for a pair with the lowest cost to cover $p_i\cup p_j$.  Using Johnson's all pairs shortest paths algorithm as a first step to compute a pairwise distance matrix for $p_i\cup p_j$, $(a^*,b^*)$ can be found in \bigO{|p_i|^3} with a memory requirement of \bigO{|p_i|^2} (as $p_i\cup p_j$ is of order $|p_i|$).  

For mobile robots with limited onboard resources, these 
requirements may be too steep.  In such circumstances, we propose instead
to sample pairs of potential new centroids when applying the algorithm.  The first
sample pair would be the agents' previous centroids, with the rest drawn at
random from the set $P_{p_i\cup p_j}$.  For $m$ sample pairs, this approach requires $2
m$ calls to Dijkstra's one-to-all shortest paths algorithm to find the pair
with the lowest cost, for a total time complexity of \bigO{m |p_i|
  \log(|p_i|)} and a memory requirement of \bigO{|p_i|}.  It is also worth
noting that if all edge weights in $G$ are equal (i.e., for a regular
grid discretization), then a breadth-first-search approach can replace
Dijkstra and the time complexity drops to \bigO{m |p_i|}.  This sample-based approach
greatly reduces the time and memory requirements, and will still
converge to a pairwise-optimal partition.

\section{Simulation Results}
\label{sec:simulations}

To demonstrate the utility of the proposed gossip coverage algorithm, we implemented it and other coverage algorithms using the open-source Player/Stage robot software system \cite{BG:08b} and the Boost Graph Library (BGL) \cite{JBS-LQL-AL:07}.  All results presented here were generated using Player 2.1.1, Stage 2.1.0, and BGL 1.34.1.  A non-convex environment was specified with a bitmap and overlaid with a $0.1m$ resolution occupancy grid, producing a lattice-like graph with all edge weights equal to $0.1m$.  To compute distances on graphs with uniform edge weights we extended the BGL implementation of breadth-first-search with a distance recorder event visitor.

Figure~\ref{fig:sim_cave} provides one example of how convergence to the set of pairwise-optimal partitions represents an improvement over Lloyd-type methods.  On the left is the non-convex environment used for all results in this Section, as well as an initial partition for $10$ robots.
The middle panel shows the result of a centralized Lloyd optimization where all agents iteratively update their centroid and partition based on the centroids of all other agents.  The final equilibrium for this centralized method is a centroidal Voronoi partition with a total cost of $624.1m$. However,
this solution is not a pairwise-optimal partition.  We started the pairwise-optimal discretized gossip coverage algorithm using this partition as the initial condition.  After $70$ pairwise
exchanges between random agent pairs, the algorithm reached the pairwise-optimal
partition at right which has a lower cost of $610.3$.  The best known $10$-partition of this map has cost $610.0$.

\begin{figure}[t]
\centering
\psfrag{Simulation count}{Simulation count}
\psfrag{Final total cost (m)}{Final total cost ($m$)}
\includegraphics[width=0.45\textwidth]{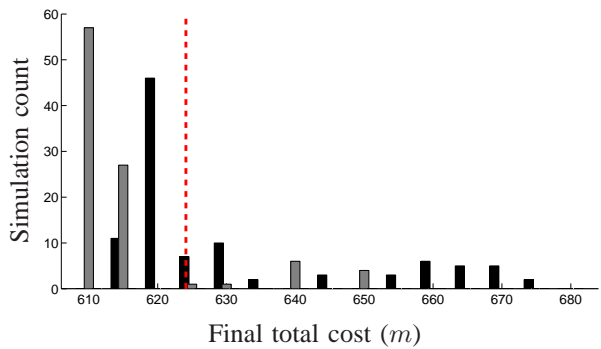}
\caption{Histograms of final total costs for $100$ simulations of pairwise-optimal (gray) and Lloyd-type (black) gossip coverage.  All runs started from the initial partition in Fig.~\ref{fig:sim_cave} with different random sequences of agent pairs.  The dashed red line shows the final cost using a centralized Lloyd algorithm.}
\label{fig:histogram}
\vspace{-0.1in}
\end{figure}

Figure~\ref{fig:histogram} compares $100$ simulations of the Lloyd-type and pairwise-optimal discretized gossip coverage algorithms.  We started
each simulation from the initial partition in Fig.~\ref{fig:sim_cave} and used
different random sequences of agent pairs for each case.  The new pairwise-optimal algorithm shows marked
improvement over the Lloyd-type gossip algorithm, achieving a total cost within $2\%$
of the best known cost in $85\%$ of trials, while also avoiding the worst local minima.

\begin{figure}[t]
\centering
\psfrag{Simulation count}{Simulation count}
\psfrag{Final total cost (m)}{Final total cost ($m$)}
\vspace{-0.25in}
\hspace{-0.10in}
\includegraphics[width=0.50\textwidth]{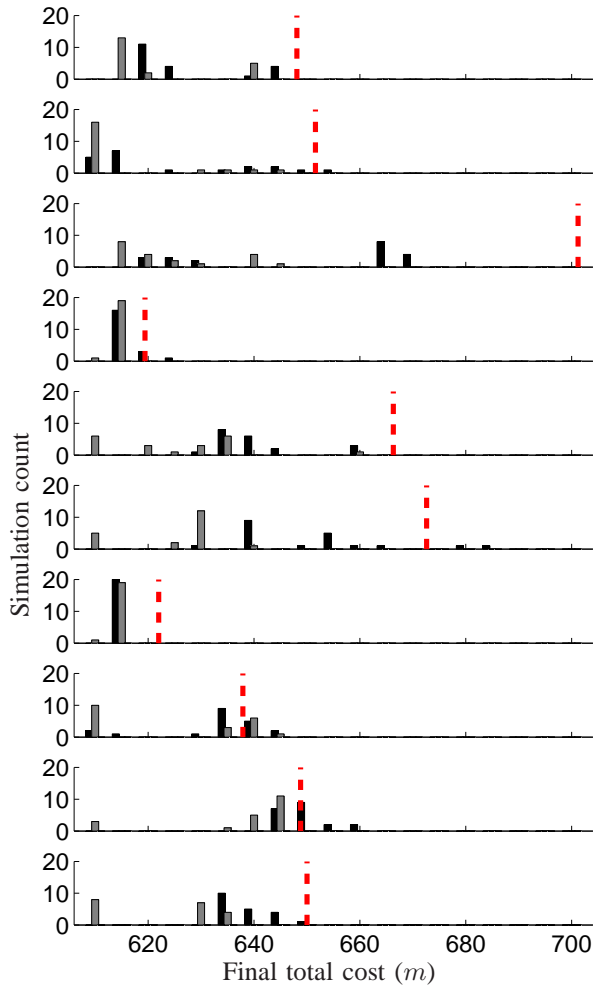}
\vspace{-0.5in}
\caption{Histograms of final costs for simulations starting from $10$ random initial conditions for the environment in Fig.~\ref{fig:sim_cave}.  The pairwise-optimal (gray) and Lloyd-type (black) gossip coverage  algorithms each ran $20$ times per initial condition using different random sequences of agent pairs.  The dashed red line shows the final cost using a centralized Lloyd algorithm.}
\label{fig:multicompare}
\vspace{-0.1in}
\end{figure}

 Figure~\ref{fig:multicompare} compares final cost
  histograms for $10$ random initial conditions for the same environment.
  Each initial condition was created by selecting unique starting locations
  for the agents uniformly at random, and using these locations to generate
  a Voronoi partition.  The histograms compare $20$ simulations of
  Lloyd-type and the new pairwise-optimal gossip coverage starting from the
  same initial partition but with different random orders of exchanges.
  The centralized Lloyd final cost for each initial condition is also
  shown.  The new pairwise-optimal method outperforms both Lloyd-type
  methods for all $10$ tests, although the Lloyd-type gossip method is
  close in two of the trials.  These results illustrate that convergence
  to a pairwise-optimal partition represents a significant performance
  enhancement over classic Lloyd methods.  Interestingly, the Lloyd-type
  gossip algorithm also substantially outperforms the centralized version
  in $8$ of the trials.  We speculate this is due to the gossip method
  taking trajectories through the space of connected $N-$partitions which
  are not possible for the centralized approach.

\addtolength{\textheight}{0.00cm}   

\section{Conclusions}
\label{sec:conclusions}

We have presented a novel distributed coverage control algorithm which requires
only pairwise communication between agents.
The classic Lloyd approach to optimizing quantizer placement involves iteration
of separate centering and Voronoi partitioning steps.
In the graph coverage domain, the separation of centroid and partition optimization 
seems unnecessary computationally for gossip algorithms.  More importantly, we have
shown that improved performance can be achieved without this separation.  
Our new pairwise-optimal discretized gossip coverage algorithm provably
converges to a subset of the centroidal Voronoi partitions
which we labeled the set of pairwise-optimal partitions.
Through numerical comparisons we demonstrated that this new subset of solutions
avoids many of the local minima in which both gossip and centralized Lloyd-type algorithms
can get stuck.

\bibliographystyle{ieeetr}
\bibliography{alias,Main,FB}

\begin{thebibliography}{10}

\bibitem{EF-NEL:06}
E.~Fiorelli, N.~E. Leonard, P.~Bhatta, D.~A. Paley, R.~Bachmayer, and D.~M.
  Fratantoni, ``Multi-{AUV} control and adaptive sampling in {Monterey Bay},''
  {\em IEEE Journal of Oceanic Engineering}, vol.~31, no.~4, pp.~935--948,
  2006.

\bibitem{PRW-RdA-MM:08}
P.~R. Wurman, R.~{D'Andrea}, and M.~Mountz, ``Coordinating hundreds of
  cooperative, autonomous vehicles in warehouses,'' {\em AI Magazine}, vol.~29,
  no.~1, pp.~9--20, 2008.

\bibitem{FB-JC-SM:09}
F.~Bullo, J.~Cort{\'e}s, and S.~Mart{\'\i}nez, {\em Distributed Control of
  Robotic Networks}.
\newblock Applied Mathematics Series, Princeton University Press, 2009.
\newblock Available at http://www.coordinationbook.info.

\bibitem{SPL:82}
S.~P. Lloyd, ``Least squares quantization in {PCM},'' {\em IEEE Transactions on
  Information Theory}, vol.~28, no.~2, pp.~129--137, 1982.
\newblock Presented as Bell Laboratory Technical Memorandum at a 1957 Institute
  for Mathematical Statistics meeting.

\bibitem{CG-JC-FB:06o}
C.~Gao, J.~Cort{\'e}s, and F.~Bullo, ``Notes on averaging over acyclic digraphs
  and discrete coverage control,'' {\em Automatica}, vol.~44, no.~8,
  pp.~2120--2127, 2008.

\bibitem{VVVa:01}
V.~V. Vazirani, {\em Approximation Algorithms}.
\newblock Springer, 2001.

\bibitem{MZ-CGC:08}
M.~Zhong and C.~G. Cassandras, ``Distributed coverage control in sensor network
  environments with polygonal obstacles,'' in {\em {IFAC} {W}orld {C}ongress},
  (Seoul, Korea), pp.~4162--4167, July 2008.

\bibitem{CHC-MZ:08}
C.~H. Caicedo-N{\`u}{\~n}ez and M.~{\v Z}efran, ``Performing coverage on
  nonconvex domains,'' in {\em {IEEE} Conf. on Control Applications}, (San
  Antonio, TX), pp.~1019--1024, Sept. 2008.

\bibitem{LCAP-VK-RCM-GASP:08}
L.~C.~A. Pimenta, V.~Kumar, R.~C. Mesquita, and G.~A.~S. Pereira, ``Sensing and
  coverage for a network of heterogeneous robots,'' in {\em {IEEE} Conf. on
  Decision and Control}, (Canc\'un, M\'exico), pp.~3947--3952, Dec. 2008.

\bibitem{OB-OB-DK-QW:07}
O.~Baron, O.~Berman, D.~Krass, and Q.~Wang, ``The equitable location problem on
  the plane,'' {\em European Journal of Operational Research}, vol.~183, no.~2,
  pp.~578--590, 2007.

\bibitem{SY-MS-DR:09}
S.~Yun, M.~Schwager, and D.~Rus, ``Coordinating construction of truss
  structures using distributed equal-mass partitioning,'' in {\em International
  Symposium on Robotics Research}, (Lucerne, Switzerland), Aug. 2009.

\bibitem{MS-DR-JJS:08}
M.~Schwager, D.~Rus, and J.~J. Slotine, ``Decentralized, adaptive coverage
  control for networked robots,'' {\em International Journal of Robotics
  Research}, vol.~28, no.~3, pp.~357--375, 2009.

\bibitem{PF-RC-FB:08u-arxiv}
F.~Bullo, R.~Carli, and P.~Frasca, ``Gossip coverage control for robotic
  networks: {D}ynamical systems on the the space of partitions,'' Dec. 2009.
\newblock Available at \texttt{http://arxiv.org/abs/0903.3642}.

\bibitem{JWD-RC-PF-FB:08z}
J.~W. Durham, R.~Carli, P.~Frasca, and F.~Bullo, ``Discrete partitioning and
  coverage control with gossip communication,'' in {\em ASME Dynamic Systems
  and Control Conference}, (Hollywood, CA), Oct. 2009.

\bibitem{BG:08b}
B.~Gerkey, R.~T. Vaughan, and A.~Howard, ``{The Player/Stage Project}.''
  \texttt{http://playerstage.sourceforge.net}, June 2008.
\newblock Version 2.11.

\bibitem{JBS-LQL-AL:07}
J.~G. Siek, L.-Q. Lee, and A.~Lumsdaine, ``{Boost Graph Library}.''
  \texttt{http://www.boost.org}, July 2007.
\newblock Version 1.34.1.

\end{thebibliography}

\end{document}